\documentclass[a4paper,11pt,fleqn]{article}
\usepackage[margin=1.25in]{geometry}
\usepackage{natbib}

\usepackage{fullpage}
\usepackage{amssymb,amsfonts,amsmath,amsthm}
\usepackage{mathrsfs}
\usepackage{amscd}
\usepackage{xcolor}
\usepackage{array}
\usepackage[linesnumbered,ruled]{algorithm2e}
\usepackage{caption}
\usepackage{subcaption}
\usepackage{enumerate}
\usepackage{doi}

\title{Extracting Weighted Automata for Approximate Minimization in Language Modelling}
\usepackage{authblk}
\date{}

\author[1,2]{Clara Lacroce\footnote{Corresponding author: clara.lacroce@mail.mcgill.ca}\footnote{The names of the authors appear in alphabetical order}}
\author[1,2]{Prakash Panangaden}
\author[2,3]{Guillaume Rabusseau}

\affil[1]{School of Computer Science, McGill University, Montr\'eal, Canada}
\affil[2]{Mila, Montr\'eal, Canada}
\affil[3]{DIRO, Universit\'e de Montr\'eal, Montr\'eal, Canada}

\newcommand{\mat}[1]{\mathbf{#1}}
\newcommand{\norm}[1]{\|#1\|}

\theoremstyle{plain}
\newtheorem{definition}{Definition}[section]

\newtheorem{theorem}{Theorem}[section]
 
\newtheorem{lemma}[theorem]{Lemma}

\theoremstyle{definition}
\newtheorem{remark}{Remark}

\usepackage{mathdots}

\newcommand{\N}{\mathbb{N}}
\newcommand{\Z}{\mathbb{Z}}
\newcommand{\R}{\mathbb{R}}
\newcommand{\C}{\mathbb{C}}
\newcommand{\A}{\mat{A}}
\renewcommand{\H}{\mat{H}}
\newcommand{\mP}{\mat{P}}
\newcommand{\mQ}{\mat{Q}}
\newcommand{\mT}{\mat{T}}

\newcommand{\balpha}{\boldsymbol{\alpha}}
\newcommand{\bbeta}{\boldsymbol{\beta}}
\newcommand{\wfa}{\langle \balpha , \{\A_a\},  \bbeta \rangle}
\newcommand{\wa}{\langle \balpha , \A,  \bbeta \rangle}

\begin{document}

\maketitle

\begin{abstract}
In this paper we study the approximate minimization problem for language modelling.  We assume we are given some language model as a black box. The objective is to obtain a weighted finite automaton (WFA) that fits within a given size constraint and which mimics the behaviour of the original model while minimizing some notion of distance between the black box and the extracted WFA. We provide an algorithm for the approximate minimization of black boxes trained for language modelling of sequential data over a one-letter alphabet. By reformulating the problem in terms of Hankel matrices, we leverage classical results on the approximation of Hankel operators, namely the celebrated Adamyan-Arov-Krein (AAK) theory. This allows us to use the spectral norm to measure the distance between the black box and the WFA. We provide theoretical guarantees to study the potentially infinite-rank Hankel matrix of the black box, without accessing the training data, and we prove that our method returns an asymptotically-optimal approximation.

\end{abstract}


\section{Introduction}
\label{sec:intro}
    
    Interpretability and high computational cost are two of the main challenges arising from the use of deep learning models~\citep{doshi-rigorous}. The need to address these issues is at the root of the increasing number of works focusing on knowledge distillation~\citep{hinton2015distilling}.
    In the case of sequential data, particular attention has been given to the problem of extracting, from a Recurrent Neural Network (RNN)~\citep{Schmidhuber}, a weighted finite automaton (WFA)~\citep{Ayache2018,Rabusseau19,WeissWFA19,Takamasa,eyraud2020,probmod,Zhang}. In fact, WFAs are a less expensive alternative to RNNs, while still being expressive and suited for sequence modelling and prediction~\citep{denishmm,cortes}.
    
    The task of knowledge distillation is closely related to the more general \emph{approximate minimization problem}, where the objective is to find a model, smaller than the original one, that imitates its behaviour while minimizing the approximation error. The advantage of doing approximate minimization instead of regular extraction is that it allows us to search for the best WFA among those of a predefined size. Since automata benefit from a graphical representation, bounding the number of states can help improve interpretability~\citep{interpretingWFA}. In this paper, we tackle the approximate minimization problem for black boxes trained for language modelling over a one-letter alphabet. We remark that, even though this is a very limited setting, it constitutes a first fundamental step towards developing provable approximation algorithms for black box models.

    A key point in solving approximation tasks is to decide how to quantify the error. We propose to rewrite the problem in terms of Hankel matrices, mathematical objects related to functions defined on sequential data. In particular, we choose to measure the error in terms of the \emph{spectral norm}, because of some of its desirable features. Indeed, the spectral norm of the Hankel matrix of a WFA can be computed in polynomial time~\citep{AAK-WFA} and we show that, similarly, minimizing the approximation error between a WFA and a black box model can be (asymptotically) solved optimally in a tractable way.
    Thus, using our method, we can measure the distance between a given RNN and the extracted WFA. This is particularly valuable, especially in light of the paper of~\citet{DistanceRNNWFA}, where the authors show that the general equivalence problem between classes of WFAs and RNNs is at best intractable, if not undecidable. The choice of this norm has the advantage that it allows us to analyze different models through their Hankel matrices, independently of the specific architecture considered. This means that addressing the approximate minimization problem using the spectral norm can facilitate the comparison between different classes of models, and the development of a distance that can be precisely computed and minimized. This is possible because Hankel matrices are at the core of the influential work of~\citet{AAK71}, which constitutes the main theoretical background on which we build our analysis. This theory has been applied before to the approximate minimization problem for WFAs, but the approach relies on the Hankel matrix considered to have known finite rank, so it cannot be directly generalized~\citep{AAK-WFA}.
    
\paragraph*{Contributions}

    The main contributions of this paper are the following:
    \begin{itemize}
        \item We present a new theoretical framework for WFA extraction from a black box trained for language modelling of sequential data over a one-letter alphabet.
        \item We use tools from control theory and arguments from random matrix theory to extend the work of~\citet{AAK-WFA} to the case of black boxes having infinite-rank Hankel matrices.
        \item We propose an algorithm that, given a black box model $\mathcal{M}$ trained for language modelling on a one letter alphabet and a target size $k$, returns a WFA with $k$ states corresponding to an asymptotically-optimal spectral approximation of $\mathcal{M}$. We do not assume any knowledge on the internal structure of the black box, nor on the training data.
        \item We propose a new way to compute the distance between a black box and the extracted WFA, based on AAK theory. We provide bounds on the approximation error in terms of spectral and $\ell^2$ norm, and strategies to improve precision when the rank is infinite.
    \end{itemize}

\section{Background}
\label{sec:background}
    We recall the definitions and basic results that we will use throughout the paper.
    
\subsection{Notation}

    Let $\N$, $\Z$ and $\R$ be the set of natural, integer and real numbers, respectively. We use bold letters for vectors and matrices; all vectors are column vectors unless otherwise specified. We denote with $\mat{v}(i)$, $\mat{M}(i,:)$ and $\mat{M}(:,j)$ the $i$-th component of the vector $\mat{v}$, and the $i$-th row and $j$-th column of $\mat{M}$, respectively. A \emph{rank factorization} of $\mat{M}\in \R^{p\times q} $ of rank $n$ is a factorization $\mat{M}=\mP\mQ$, with $\mP \in \R^{p\times n}$, $\mQ \in \R^{n\times q}$, with $\mP$, $\mQ$ of rank $n$. Let $\mat{M} \in \R^{p \times q}$ of rank $n$, the compact \emph{singular value decomposition}~(SVD) of $\mat{M}$ is $\mat{M}=\mat{U}\mat{D}\mat{V}^{\top}$, where $\mat{U}\in \R^{p\times n}$, $\mat{D}\in \R^{n\times n}$, $\mat{V}\in \R^{q \times n}$, with $\mat{U}^{\top}\mat{U}=\mat{V}^{\top}\mat{V}=\mat{1}$, and $\mat{D}$ is diagonal. The columns of $\mat{U}$ and $\mat{V}$ are called left and right \emph{singular vectors}, while the entries $\sigma_0 \geq \dots \geq \sigma_{n-1} > 0$ of $\mat{D}$ are the \emph{singular values}. The \emph{Moore-Penrose pseudo-inverse} $\mat{M}^+$ of $\mat{M}$ is the unique matrix such that $\mat{M}\mat{M}^+\mat{M}=\mat{M}$, $\mat{M}^+\mat{M}\mat{M}^+=\mat{M}^+$, with $\mat{M}^+\mat{M}$ and $\mat{M}\mat{M}^+$ Hermitian.

    A \emph{Hilbert space} is a complete normed vector space where the norm arises from an inner product.
    Let $X$, $Y$ be Hilbert spaces. A linear operator $T: X \rightarrow Y$ is \emph{bounded} if it has finite \emph{operator norm}, \emph{i.e.} $\norm{T}_{op} = \sup_{\norm{g}_X\leq 1}\norm{Tg}_Y<\infty$, while it is \emph{compact} if it is the limit of finite rank operators in the operator norm. When $T$ is the limit in the operator norm of the sequence of operators $\{T^i\}_{i\geq 0}$, we use $T^i\to T$ to denote the limit: $\lim_{i\to\infty}\norm{T^i-T}=0$. Let $T:X \rightarrow Y$ be a compact operator, the \emph{adjoint} $T^*$ is the linear operator $T^*:Y \rightarrow X$ such that $\langle Tx,y \rangle_Y=\langle x,T^*y\rangle_X$, where $\langle \cdot,\cdot \rangle$ is the inner product of the corresponding Hilbert space, $x\in X$, $y\in Y$. The \emph{singular numbers} $\{\sigma_n\}_{n \geq 0}$ of $T$ are the square roots of the eigenvalues of $T^* T$, arranged in decreasing order. A singular number is \emph{simple} if it is not repeated. Let $\mat{T}$ be the infinite matrix associated with $T$ by some canonical orthonormal basis. The Hilbert-Schmidt decomposition generalizes the compact SVD for the infinite matrix of a compact operator $T$, using singular numbers and orthonormal vectors:
    $\mT\mat{x}=\sum_{n\geq0}\sigma_n\langle\mat{x},\boldsymbol{\xi}_n \rangle \boldsymbol{\eta}_n$ 
    (see~\citet{Zhu}). The \emph{spectral norm} $\norm{\mat{T}}$ of the matrix of the operator $T$ is the largest singular number, and corresponds to the operator norm of $T$. 
    
    Let $\mathbb{T}=\{z\in \C: |z|=1\}$ and $\mathbb{D}=\{z\in \C: |z|<1\}$ be the complex unit circle and the complex unit disc, respectively, $z$ the complex variable. Let $p>1$, we denote with $\mathcal{L}^p(\mathbb{T})$ the space of measurable functions on $\mathbb{T}$ with the property that the $p$-th power of their absolute value is Lebesgue integrable.

\subsection{Hankel Matrix and WFAs}
\label{subsection:WFA}
    
    Let $\Sigma$ be a fixed finite alphabet and $\Sigma^*$ be the set of all finite strings with symbols in $\Sigma$. We denote with $\varepsilon$ the empty string, and $\Sigma' = \Sigma\cup\{\varepsilon\}$. Given $p,s \in \Sigma^*$, we denote with $ps$ their concatenation. Let $f : \Sigma^* \to \R$ be a function defined on sequences, we can consider a bi-infinite matrix $\H_f \in \R^{\Sigma^* \times \Sigma^*}$ having rows and columns indexed by strings and defined by $\H_f(p,s) = f(ps)$ for $p, s \in \Sigma^*$.
    
    \begin{definition}
         A (bi-infinite) matrix $\H \in \R^{\Sigma^* \times \Sigma^*}$ is \textbf{Hankel} if for all $p, p', s, s' \in \Sigma^*$ such that $p s = p' s'$, we have $\H(p,s) = \H(p',s')$. Given a Hankel matrix $\H \in \R^{\Sigma^* \times \Sigma^*}$, there exists a unique function $f : \Sigma^* \to \R$ such that $\H_f = \H$.
    \end{definition}
    
    Weighted finite automata are a class of models defined over sequential data. A \emph{weighted finite automaton} (WFA) of $n$ states over the alphabet $\Sigma$ is a tuple $A = \wfa$, where $\balpha,$ $\bbeta \in \R^n$ are the vector of initial and final weights, respectively, and $\A_a \in \R^{n \times n}$ is the matrix containing the transition weights associated with each symbol $a\in\Sigma$. While WFAs can in general be defined over semirings, we will only consider automata with real weights. In this case, every WFA $A$ computes a function $f_A : \Sigma^* \to \R$, \emph{i.e.}, given a string  $x = x_1 \cdots x_t \in \Sigma^*$, it returns $f_A(x) = \balpha ^\top \A_{x_1} \cdots \A_{x_t} \bbeta = \balpha ^\top \A_x \bbeta$. We say that the function $f: \Sigma^* \to \R$ is \emph{rational} if there exists a WFA $A$ with $f=f_A$, and the \emph{rank} of $f$ is the size of the smallest WFA computing $f$. We can use the Hankel matrix $\H_f$ to recover information about the WFA.
    
    \begin{theorem}[\citet{CP71,Fli}]\label{fliess}
        A function $f:\Sigma^* \to \R$ can be computed by a WFA if and only if the corresponding Hankel matrix $\H_f$ has finite rank $n$. In that case, $n$ is the minimal number of states of any WFA computing $f$.
    \end{theorem}
    
    Given a Hankel matrix $\H_f$ of rank $n$, we can recover the minimal WFA $A$ computing $f$ by using the method proposed in~\citet{BalleCLQ14}, an efficient \emph{spectral algorithm} which is robust to noise. In particular, we can consider a basis  $\mathcal{B}=(\mathcal{P},\mathcal{S})$, with $\mathcal{P},\mathcal{S} \subset \Sigma^*$, and a sub-block $\H_{\mathcal{B}}$ of $\H_f$ defined over $\mathcal{B}$. The method can be applied whenever $\mathcal{B}$ is \emph{prefix-closed} and \emph{complete}, \emph{i.e.}, when $\mathcal{P}=\mathcal{P}'\cdot\Sigma'$ for some $\mathcal{P}'$, and $\H_{\mathcal{B}}$ has rank $n$. In this case, we can consider the sub-block $\H_{a}$ defined over $\mathcal{B}$ by $\H_{a}(u,v)=\H(u\cdot a,v)$ for each $a \in \Sigma'$, 
    and the vectors $\mat{h}_{\mathcal{P},\varepsilon},\, \mat{h}_{\varepsilon,\mathcal{S}}$ having coordinates $\mat{h}_{\mathcal{P},\varepsilon}(u)=\H(u,\varepsilon)$ and $\mat{h}_{\varepsilon,\mathcal{S}}(v)=\H(\varepsilon,v)$.
    Then, from the rank factorization $\H_{\varepsilon}=\mat{P}\mat{S}$ we can compute a minimal WFA $A= \wfa$ for $f$:
    \begin{equation}\label{eq:spectralmethod}
        \balpha^\top= \mat{h}_{\varepsilon,\mathcal{S}}^\top \mat{S}^+, \quad \bbeta=\mat{P}^+ \mat{h}_{\mathcal{P},\varepsilon}, \quad \A_{a}=\mat{P}^+\H_{a}\mat{S}^+.
    \end{equation}

\subsection{Recurrent Neural Networks}

    Recurrent Neural Networks~\citep{Schmidhuber}, or RNNs, are a class of neural networks designed to process sequential data. Unlike feedforward neural networks, RNNs maintain an internal memory based on history information through the hidden states. At each timestep, a RNN receives an input and returns a new state vector, depending on the input and on the sequence received so far. There exists several types of architectures for these models, which makes them well suited for a variety of tasks~\citep{WeissGY18,merrill2020}. Analogously to~\citet{Ayache2018} and \citet{WeissWFA19}, we focus on LM-RNNs, where the RNN is trained for \emph{language modelling}, and the task is to predict the next element in a sequence. Thus, a LM-RNN can be seen as computing the probability associated to a string, and it defines a distribution over sequences that can be represented by a Hankel matrix.

\subsection{AAK Theory}
\label{subsection:AAK}

    The key idea behind our method is that, since a model computing $f:\Sigma^* \rightarrow \R$ corresponds to a Hankel matrix $\H=\H_f$, the minimization problem can be reformulated using Hankel matrices. The objective becomes to find a Hankel matrix $\mat{G}$ that approximates $\H$ optimally in the spectral norm, and then extract a WFA from it. This approach has been explored before by~\citet{AAK-WFA}, but their method does not generalize to infinite-rank Hankel matrices. We recall a well known result in low-rank matrix approximation. 
    \begin{theorem}[\citet{Eckart}]\label{thm:eckart}
        Let $\H$ be a Hankel matrix of rank $n$, and let $\sigma_0 \geq \dots \geq \sigma_{n-1}>0$ be the sequence of its singular numbers. Then, if $\mat{R}$ is a matrix having rank $k$, we have:
        \begin{equation}
            \norm{\H - \mat{R}}\geq \sigma_k
        \end{equation}
        and the minimum is attained when $\mat{R}$ is the truncated SVD of $\H$. 
    \end{theorem}
    Unfortunately this result does not solve our problem, since truncating the SVD does not necessarily produce a Hankel matrix, which is required to recover a WFA. When $|\Sigma|=1$, the issue can be solved using a theory of optimal approximation called Adamyan-Arov-Krein (AAK) theory~\citep{AAK71}, which allows us to search for the best approximation directly in the set of finite-rank Hankel matrices. In order to introduce AAK theory, for the rest of this section we will assume $|\Sigma|=1$. The same assumption will be required in the contribution (for more details we refer the reader to Section~\ref{assumptions}). When the alphabet only has one letter, we can denote a string with the number corresponding to how many times the single character is repeated (\emph{e.g.} $'aaa'=3$), and we can identify $\Sigma^*$ with $\N$. Let $\ell^2$ be the Hilbert space of square-summable sequences over $\N$. We interpret the  Hankel matrix $\H_f$ associated to $f:\N \rightarrow \R$ as the expression, in terms of the canonical basis, of a linear Hankel operator $H_f:\ell^2 \rightarrow \ell^2$. To reformulate the problem in the setting of AAK theory, we embed $\ell^2$ into $\ell^2(\Z)$, and apply the Fourier isomorphism to associate a complex function to each sequence in $\ell^2(\Z)$. In fact, a function  $\phi(z) \in \mathcal{L}^2(\mathbb{T})$ in the complex variable $z$ can be represented by its Fourier expansion
    \begin{equation}
        \phi(z)=\sum_{n \in \Z}\widehat{\phi}(n)z^n
    \end{equation}
    and can be identified, using the orthonormal basis $\{z^n\}_{n \in \Z}$, with the sequence of its Fourier coefficients 
    \begin{equation}
        \widehat{\phi}(n)= \int_{\mathbb{T}}\phi(z) \bar{z}^n dz, \, n \in \Z.
    \end{equation}
    Note that in this paper we will work with two classes of functions, functions over sequences and complex functions. To avoid any confusion we will make explicit the dependence on the complex variable $z=e^{it}$.
    
    We can now partition the function space $\mathcal{L}^2(\mathbb{T})$ into two subspaces.
    
    \begin{definition}\label{def}
        For $0<p\leq\infty$ , the \textbf{Hardy space} $\mathcal{H}^p$ and the \textbf{negative Hardy space} $\mathcal{H}^p_-$ on $\mathbb{T}$ are the subspaces of $\mathcal{L}^p(\mathbb{T})$ defined as:
        \begin{align}
            \mathcal{H}^p&= \{ \phi(z) \in \mathcal{L}^p(\mathbb{T}) : \widehat{\phi}(n)=0, n < 0\}, \\
            \mathcal{H}^p_-&=\{ \phi(z) \in  \mathcal{L}^p(\mathbb{T}) : \widehat{\phi}(n)=0, n \geq 0\}. \notag
        \end{align}
    \end{definition}
    Since the elements of the Hardy space $\mathcal{H}^p$ can be canonically identified with the set of $p$-integrable functions analytic in $\mathbb{D}$, we will make no difference between these functions in the complex unit disc and their boundary value on the complex unit circle~\citep{Nikolski}. 
    
    It is possible to characterize Hankel operators using Hardy spaces (more details can be found in~\citet{Nikolski}). Let $\mathbb{P}_-:\mathcal{L}^2(\mathbb{T}) \rightarrow \mathcal{H}^2_- $ be the orthogonal projection on the negative Hardy space.
    \begin{definition}\label{Hankel2}
        Let $\phi(z)$ be a function in $\mathcal{L}^2(\mathbb{T})$. A \textbf{Hankel operator} is an operator $H_{\phi}:\mathcal{H}^2 \rightarrow \mathcal{H}^2_-$ defined by $ H_{\phi}f(z)=\mathbb{P}_-\phi f(z)$. The function $\phi(z)$ is a \textbf{symbol} for $H_{\phi}$.
    \end{definition}
    
    From now on, Hankel operators will always be interpreted in Hardy spaces.
    
    We recall that a complex function $\phi(z)$ is \emph{rational} if $\phi(z)=p(z)/q(z)$, with $p(z)$ and $q(z)$ polynomials, and it is \emph{strictly proper} if the degree of $p(z)$ is strictly smaller than that of $q(z)$. Finite rank Hankel operators are closely related to the theory of rational functions.
    \begin{theorem}[\citet{kronecker}]\label{theorem:Kronecker}
        Let $H_{\phi}$ be a bounded Hankel operator with matrix $\H$. Then $\H$ has finite rank if and only if $\mathbb{P}_-\phi$ is a strictly proper rational function. Moreover the rank of $\H$ is equal to the number of poles, counted with their multiplicities, that the function $\mathbb{P}_-\phi$ has inside $\mathbb{D}$.
    \end{theorem}
    
    In this remark we highlight the two interpretations of Hankel matrices used in the paper. For an example, we refer the reader to Appendix~\ref{appendix:ex}.
    \begin{remark}\label{remark1}
        On the one hand we can consider the matrix $\H$ with respect to the basis of the sequence space $\ell^2$, and associate $\H$ with the function $f:\N \rightarrow \R$. In this case $\H(i,j)=f(i+j)$ for $i,j\geq0$. On the other hand, we can look at $\H$ with respect to the standard orthonormal bases of $\mathcal{H}^2$ and $\mathcal{H}^2_-$. Now, $\H$ is associated with a complex function $\phi(z)\in \mathcal{L}^2(\mathbb{T})$, and we have $\H(j,k)= \widehat{\phi}(-j-k-1)$. Note that $f$ and $\phi$ are related through the Fourier isomorphism, and we have: $f(n)=\widehat{\phi}(-n-1)$.
    \end{remark}
    
    The core result of AAK theory is that, when minimizing a compact Hankel operator, the constraint of preserving the Hankel property does not affect the quality of the approximation. This means that we can attain the same error as the truncated SVD while searching for the optimal approximation within the class of Hankel matrices.
  
    \begin{theorem}[\citet{AAK71}]\label{theorem:aakop}
        Let $H$ be a compact Hankel operator, with matrix $\H$ of rank $n$ and singular numbers $\sigma_0 \geq \dots \geq \sigma_{n-1}>0$. Then, there exists a unique Hankel operator $G_k$ of rank $k<n$ such that
        \begin{equation}
            \norm{\H - \mat{G}_k}= \sigma_k.
        \end{equation}
        We say that $G_k$ is the optimal approximation.
    \end{theorem}

    The following theorem, based on the constructive proof for Theorem~\ref{theorem:aakop}, can be used to find a symbol of the best approximation. We recall that a $\sigma$-\emph{Schmidt pair} $\{\boldsymbol{\xi}, \boldsymbol{\eta}\}$ for $H$ is a couple of vectors such that: $\mat{H}\boldsymbol{\xi}=\sigma \boldsymbol{\eta}$ and $\mat{H}^*\boldsymbol{\eta}= \sigma\boldsymbol{\xi}$, and that a \emph{Toeplitz matrix} is a matrix $\mat{T}$ with entries defined by $\mat{T}(i,j)=t_{j-k}$ for $j,k\geq0$.
    \begin{theorem}[\citet{discreteH}]\label{toeplitz:unimodular}
        Let $\{\boldsymbol{\xi}_k, \boldsymbol{\eta}_k\}$ be any $\sigma_k$-Schmidt pair for $H$. We consider a bi-infinite upper triangular Toeplitz matrix $\mat{T}$, defined as follows:
        \begin{itemize}
            \item $\mat{T}$ has only zeros on the main diagonal
            \item the first row is defined by $\mat{T}(0,k)=\H(0,k-1)$ for $k>0$
            \item the remaining entries are defined by $\mat{T}(j,k)=\mat{T}(j+1,k+1)$.
        \end{itemize}
        Let $\mat{z}=\begin{pmatrix} 1 & z & z^2 & \dots \end{pmatrix}^{\top}$ where $z$ is the complex variable. Then, the rational function $r(z)$ corresponding to a symbol for the best approximation of rank $k$ is:
        \begin{equation}\label{eq:toeplitz}
            r(z)=\mathbb{P}_-\left(\frac{\mat{z}^{\top}\mat{T}\boldsymbol{\xi}}{\mat{z}^{\top}\boldsymbol{\xi}}\right).
        \end{equation}
    \end{theorem}
    For an example of the matrix $\mat{T}$, we refer the reader to Equation~\ref{eq:matrix}.
    
    We conclude emphasizing the important relation between matrix and operator.
    \begin{remark}\label{remark}
        A Hankel matrix $\H$ can be seen as the representation of a Hankel operator $H$ by means of a canonical basis. For example, as noted in Remark~\ref{remark1}, a Hankel matrix can be interpreted as an operator acting between sequences or between Hardy spaces, depending on the basis used. While we are interested in matrices, most of the results we use are obtained in the context of operators. To apply AAK theory, we choose to work with the basis of the Hardy spaces. This way, we have associated with the matrix a unique operator, the one defined in Definition~\ref{def}. This allows us to alternate between matrix and operator, directly transferring results from one interpretation to the other. Moreover, if $H$, $G$ are operators with matrices $\H$, $\mat{G}$, we recall that:
        \begin{equation}
            \norm{H-G}=\norm{\H-\mat{G}},
        \end{equation}
        where on the left we are considering the operator norm, and on the right the spectral one. Thus, while we keep the notations distinct to remain faithful to the original definitions (\emph{e.g.}, compactness is a property defined for $H$, not for $\H$), to have an intuition of the results it is always possible to think in terms of Hankel matrices.
    \end{remark}

\section{Asymptotically-Optimal Approximate Minimization}
    
    We are now ready to introduce the main contribution of this paper.

\subsection{Problem Formulation}

    We recall that a bounded operator $H$ is compact if and only if there exists a sequence of finite rank operators $\{H^i\}_{i\geq0}$ converging to it, \emph{i.e.} if $H^i\to H$. Let $G_k$ and $G_k^i$ be the rank $k$ optimal approximations to $H$ and $H^i$, respectively, according to Theorem~\ref{theorem:aakop}. 
    We say that the sequence of matrices $\{\mat{G}^i_k\}_{i\geq0}$ is an \emph{asymptotic sequence} for the matrix $\mat{G}_k$, if the corresponding sequence of operators $\{G^i_k\}_{i\geq0}$ converges to the operator $G_k$, \emph{i.e.}, if $G^i_k\to G_k$.
    Note that, if $\{\sigma_j\}_{j\geq 0}$ are the singular numbers of $H$, for an asymptotic sequence we have:
    \begin{equation}\label{eq:asympt}
        \lim_{i\to\infty}\norm{H - G_k^i}=\sigma_k.
    \end{equation}
    
    We can now formally define the approximation problem. Let $|\Sigma|=1$, $\Sigma^*=\N$. We consider a LM-RNN computing a function $f:\N \rightarrow \R$, with Hankel matrix $\H$ corresponding to the operator $H$. Let $k$ be the target size of the approximation, and $n>k$. We denote with $G_k$ the optimal rank $k$ approximation of $H$. We say that a WFA $\widehat{A}_k^n$ with $k$ states is an \emph{asymptotically-optimal $(n,k)$-approximation} for the LM-RNN if the Hankel matrix $\mat{G}_k^n$ of $\widehat{A}_k^n$ belongs to an asymptotic sequence for $\mat{G}_k$. In the notation, we will omit the specification $(n,k)$ whenever is clear from the context, or not relevant.
    
    Intuitively, we can consider a sequence of finite rank matrices $\{\H^i\}_{i\geq0}$ converging to $\H$, and associate to each of them a WFA (Theorem~\ref{theorem:Kronecker}). This means that we have a sequence of WFAs of increasing size that ``converges'' to the LM-RNN. The matrix $\mat{G}_k$ of rank $k$ corresponds to the optimal approximation for $\H$, \emph{i.e.}, it is the WFA $\widehat{A}_k$ with $k$ states that best approximate the LM-RNN. Now, from the sequence of matrices $\mat{G}^i_k$ of optimal rank $k$ approximations, we obtain a second sequence of WFAs $\widehat{A}^i_k$, all having size $k$. When $\{\mat{G}^i_k\}_{i\geq0}$ is an asymptotic sequence for $\mat{G}_k$, the corresponding sequence of WFAs ``converges'' to $\widehat{A}_k$.

    We will study the convergence of asymptotic sequences in the next section. In particular, we will prove that a solution for the asymptotically-optimal problem can be obtained from Theorem~\ref{theorem:aakop}, but it is not unique, since different sequences $\{\H^i\}_{i\geq0}$ lead to different approximations. Nonetheless, we will show that we can get arbitrarily close to the optimal solution. 
    
    We briefly remark that it is possible to consider an alternative formulation of the approximate minimization problem~\citep{KungLin}. In this case, instead of fixing the size of the approximation, we set the tolerance allowed for the approximation error. Thus, the objective becomes to find the smallest possible WFA such that the spectral norm of the approximation error is smaller than a fixed constant $\rho$ of choice. In this case, if $\rho\in (\sigma_{k},\sigma_{k-1})$, then the best approximation has size at least $k-1$, and can be found following the same solution we will present for the standard approximation problem. 
    
    \subsection{Assumptions}
    \label{assumptions}
 
    The main limitation of this approach is that the results outlined in Section~\ref{subsection:AAK} can be applied only if $|\Sigma|=1$. In this case, $\Sigma^*$ can be identified with $\N$, and canonically embedded into $\Z$. This fundamental step allows us to use the Fourier isomorphism to reformulate the problem in the Hardy space, where it can be solved using Theorem~\ref{theorem:aakop}. If $|\Sigma|>1$, $\Sigma^*$ is a free non-abelian monoid, therefore it cannot be embedded into $\Z$.
    Therefore, for the rest of the paper we will assume $|\Sigma|=1$, and identify $\Sigma^*=\N$.
    
    We remark that the proof of Theorem~\ref{theorem:aakop} is constructive only for compact operators. We will show that compactness is automatically respected by LM-RNNs (Theorem~\ref{thm:wiener}) and that the necessary condition is actually less restrictive. In fact, if $f$ is the function computed by the black box considered, it is enough that $f\in\ell^1$. Thus, even though we mainly refer to LM-RNNs, the proposed algorithm can be applied to any black box for language modelling on a one-letter alphabet, for example transformers~\citep{transformers}.

\subsection{Compactness of the Hankel Matrix}
\label{sec:compactness}

    To apply the results of Section~\ref{subsection:AAK}, we need to find a way to test for compactness. This is the main theoretical challenge addressed by the paper. In fact, with matrices of known finite rank, like in the case of WFAs, compactness is achieved by requiring $f\in \ell^2$~\citep{Balle19}. Then, the problem can be rewritten in terms of finite matrices, the Gramians, and it is possible to find an algorithm returning the parameters of the unique best approximating WFA~\citep{AAK-WFA}. Instead, in the case of LM-RNNs we don't have access to the full bi-infinite Hankel matrix $\H$, and the unknown rank might not be finite. Therefore, there is no guarantee that the problem can be solved algorithmically.
    
    As noted before, the operator $H$ is compact if and only if there is a sequence of finite rank operators $\{H^i\}_{i\geq0}$, with $H^i\to H$. Given such converging sequence, the key idea is to find the optimal approximation of rank $k$ for each of its element. This is possible because every $H^i$ is an operator having known finite rank, for which we know how to algorithmically solve the approximation problem. 
    It remains to ensure the continuity of the approximation: if $G_k$ and $G_k^i$ are the optimal approximations of $H$ and of $H^i$, respectively, we want $\{\mat{G}_k^i\}_{i\geq0}$ to be an asymptotic sequence for $\mat{G}_k$, so that $G_k^i \to G_k$.
    This problem has been analyzed, for signal processing, in the fundamental work of~\citet{chuisystem_red} and \citet{chui_cont}. We recall the following result.
    \begin{theorem}[\citet{chui_cont}]\label{theorem:convergence}
        Let $H$ be a bounded Hankel operator, $\{\sigma_i\}_{i\geq0}$, its singular numbers. Suppose to have a sequence $\{H^i\}_{i\geq0}$ of bounded Hankel operators such that $H^i \to H$. Let $G_k$ and $G_k^i$ be the unique optimal approximations of rank $k$ of $H$ and of $H^i$ for any $i$, respectively. If $\sigma_{k-1} \neq \sigma_k$, then the sequence $\{G_k^i\}_i\geq0$ converges to $G_k$.
    \end{theorem}
    This theorem gives us the conditions under which we can solve the approximation problem (at least asymptotically) for the matrix $\H(i,j)=f(i+j)$ of the LM-RNN. 
    
    The first step is to find a converging sequence of operators (or, as seen in Remark~\ref{remark}, a converging sequence of matrices). We can define one by truncation: let $t\geq0$, we consider the sequence of matrices defined as:
    \begin{equation}\label{eq:trunc}
        \H^t(i,j)=\begin{cases}
            f(i+j) &\text{ if } i+j\leq t\\
            0 &\text{otherwise}
        \end{cases}
    \end{equation}
    (see Equation~\ref{eq:matrix} for an example).
    We have the following theorem:
    \begin{theorem}~\label{thm:wiener}
        Let $|\Sigma|=1$. Let $f:\N \rightarrow \R$ be the function computed by a black box for language modelling, $\H$ the Hankel matrix. Let $\{\H^t\}_{t\geq0}$ as in Equation~\ref{eq:trunc}. Then, since $f\in\ell^1$, we have that $H^t\to H$.
    \end{theorem}
    \begin{proof}
        Let $f:\N \rightarrow \R$ be the function computed by the black box. We have:
        \begin{equation}\label{eq:tailbound}
            \norm{H - H^t}\leq \left\lVert \sum_{i=0}^{\infty}f(i)z^{-i-1} - \sum_{i=0}^{t} f(i)z^{-i-1} \right\rVert_{\infty} \leq \left\lVert \sum_{i=t+1}^{\infty} f(i)z^{-i-1}\right\rVert_{\infty} \leq \sum_{i=t+1}^{\infty} |f(i)|
        \end{equation}
        where the first inequality follows from Theorem~\ref{thm:nehari}. Since the black box is trained for language modelling, we have that $\sum_{k\geq 0}|f(k)|=1$. Thus, $f\in \ell^1$, and it follows directly that $H^t\to H$.
    \end{proof}
    We remark that the proof relies only on $f \in \ell^1$. Note that we have found a sequence of finite rank operators converging to $H$, therefore $\H$ is compact.
    
    The second step is to ensure that the property $\sigma_{k-1}\neq\sigma_{k}$ on the singular numbers of $H$ holds when $k$ is the size of the best approximation. This condition cannot be tested experimentally, since we don't have access to the infinite Hankel matrix $\H$. Instead, we can address the problem by using arguments from random matrix theory. In fact, up to at worst a small perturbation, we can view any $\H^t$ for $t>0$ as a random matrix having only simple singular values with probability one, and this property holds (in limit) also for $\H$~\citep{vonNeumann,tao-random}. Note that compact operators have simple spectrum after arbitrarily small perturbations, which do not have a big effect on the quality of the result since the spectrum of symmetric matrices is very stable~\citep{Hormander,kato,taobook}. In practice, for most settings the Hankel matrix $\H$ will satisfy the condition of Theorem~\ref{theorem:convergence} with probability one. This is the case, for example, of RNNs trained using a gradient based method with a random initialization. On the other hand, since we want to keep our analysis as general as possible, we also need to consider an adversarial setting, in which the black box to approximate is specifically chosen to have $\sigma_{k-1}=\sigma_k$. To avoid this kind of situation we can add some random noise to the matrix $\H$ post training. To preserve compactness, it is important to choose the Hankel matrix of noise $\mat{N}$ appropriately. For instance, $\mat{N}$ can be a Hankel matrix, with first row $\mat{N}(0,j)$ sampled uniformly in the interval $\left[-(j+2)^{-p},(j+2)^{-p}\right]$, with $p\geq2$ fixed, so that the operator $N$ is compact. Moreover, for every $\varepsilon>0$, we can find an exponent $p\geq2$ such that $\norm{\mat{N}}\leq\varepsilon$, so the perturbation can be chosen to be arbitrarily small. Note that $\H+\mat{N}$ is then a random matrix corresponding to a compact Hankel operator, and satisfies the conditions of Theorem~\ref{theorem:convergence} with probability one. 
    We will address the additional error due to small perturbations in Section~\ref{sec:error}.
    
    We are finally ready to show that if $\H^n$ belongs to the sequence of bi-infinite truncation matrices $\{\H^t\}_{t\geq0}$ introduced in Equation~\ref{eq:trunc}, an asymptotically-optimal $(n,k)$-approximation can be found by solving the problem described by Theorem~\ref{theorem:aakop} for $\H^n$. 
    \begin{theorem}\label{thm:nearopt}
        Let $\H$ and $\H^n$ be as above, and assume $\sigma_k\neq\sigma_{k-1}$. If $\mat{G}_k^n$ is the optimal approximation of $\H^n$ according to Theorem~\ref{theorem:aakop}, then a WFA having Hankel matrix $\mat{G}_k^n$ is an asymptotically-optimal $(n,k)$-approximation, and we have:
        \begin{equation}\label{eq:error}
            \sigma_k \leq \norm{\H-\mat{G}_k^n} \leq \sigma_k+ 2\left(1-\sum_{i=0}^{n}f(i)\right).
        \end{equation}
    \end{theorem}
    \begin{proof}
        Let $\sigma^n_k$ be the singular number $k+1$ of the operator $H^n$, and let $\mat{G}_k^n$ be the optimal approximation described by Theorem~\ref{theorem:aakop}, \emph{i.e.} $\norm{\H^n-\mat{G}_k^n}=\sigma^n_k$. We have:
        \begin{equation}\notag
            \norm{\H-\mat{G}_k^n} \leq \norm{\H-\H^n}+\norm{\H^n-\mat{G}_k^n}= \norm{\H-\H^n}+\sigma^n_k.
        \end{equation}
        From Theorem~\ref{thm:eckart} we know that $\norm{\H-\mat{G}_k^n}\geq \sigma^n_k$. On the other hand, using Lemma~\ref{lemma:singnumb} and Cauchy's interlace theorem~\citep{cauchy}, we obtain $\sigma^n_k \leq \sigma_k + \norm{\H-\H^n}$. It follows that:
        \begin{equation}
            \sigma_k \leq \norm{\H-\mat{G}_k^n} \leq \sigma_k+ 2\norm{\H-\H^n}.
        \end{equation}
        Now, $\H^n$ belongs to the sequence of truncating matrices $\{\H^t\}_{t\geq0}$, and $H^t\to H$ (Theorem~\ref{thm:wiener}).
        Since $\sigma_k\neq\sigma_{k-1}$, the conditions of Theorem~\ref{theorem:convergence} hold. Therefore, the sequence of matrices of best approximations $\{\mat{G}_k^t\}_{t\geq0}$ is an asymptotic sequence for $\mat{G}_k$, and $\mat{G}_n^k$ belongs to it. Thus, the WFA having matrix $\mat{G}_n^k$ is an asymptotically-optimal $(n,k)$-approximation, and Equation~\ref{eq:asympt} holds.
        Moreover, from Equation~\ref{eq:tailbound}, we have:   
        \begin{equation}\label{eq:probbound}
        \norm{\H-\mat{G}_k^n}\leq \sigma_k+ 2\norm{\H-\H^n}\leq \sigma_k+ 2\sum_{i=n+1}^{\infty}f(i)=\sigma_k+ 2\left(1-\sum_{i=0}^{n}f(i)\right).
        \end{equation}
    \end{proof}
    The bound clearly shows that, as $n$ increases, we approach the optimal approximation.
    
\subsection{Algorithm}
    
    When the rank $r$ of $\H$ is finite and known, it is possible to find directly the optimal approximation. This can be done by first extracting a WFA of size $r$ from the LM-RNN~\citep{Ayache2018}, and then applying the algorithm of~\citet{AAK-WFA} to obtain the unique optimal WFA. Therefore, in our algorithm we focus on the case in which the rank $r$ is unknown, and look for an asymptotically optimal approximation. This entails assuming that the truncation $\H^n$ has full rank: if this was not the case, since $\H^n$ is the leading principal submatrix of $\H$, we would have $r=\operatorname{rank}(\H^n)$~\citep{minor}.
  
    To simplify the notation across this section, we set $f_i=f(i)$. We recall the two bi-infinite matrices necessary to find the best approximation:
    \begin{equation}\label{eq:matrix}
        \H^n=    \begin{pmatrix} f_0 & f_1 & \dots & f_{n-1} & 0&\dots\\
                               f_1 &  &\iddots & \iddots & \vdots&\\
                               \vdots &\iddots & \iddots& &\vdots&\\
                               f_{n-1} & \iddots & & &\vdots&\\
                               0 & \dots & \dots &\dots & 0&\\
                               \vdots&&&&&\ddots
            \end{pmatrix}, \quad \quad
        \mat{T}=\begin{pmatrix} 0 & f_0 & \dots & f_{n-1} & 0&\dots\\
                               \vdots & \ddots &\ddots &  & f_{n-1}&\\
                               \vdots &&\ddots& \ddots&\vdots&\\
                               \vdots & & & \ddots&f_0&\\
                               0 & \dots & \dots &\dots & 0&\\
                               \vdots&&&&&\ddots
            \end{pmatrix}.
    \end{equation} 
    The key to successfully implement Theorem~\ref{theorem:aakop}, which applies only to infinite matrices, is in the definition of the truncation.
    In fact, this allows us to discard the zero-part and work only with the $n\times n$ sub-block of $\H^n$, which we will still denote with $\H^n$ for the sake of simplicity.
    Analogously, if $\mat{z}$ and $\mat{T}$ are the infinite vector and matrix defined in Theorem~\ref{toeplitz:unimodular} for $\H$, in the algorithm we will consider the truncations associated to $\H^n$:
     \begin{equation}\label{eq:trunc_alg}
        \mat{z}^n(i)=\mat{z}(i), \quad \mat{T}^n(i,j)=\mat{T}(i,j)\quad \text{for }i,j <n,\quad\quad \mat{z}^n\in \R^{n\times 1}, \mat{T}^n\in \R^{n\times n}
    \end{equation}
    where the discarded entries are irrelevant, being multiplied by zeros in the infinite case.

    We can finally analyze the building blocks of Algorithm~\ref{alg:approx}.

    \begin{algorithm}[t]
    \caption{\texttt{AAKmethod}}\label{alg:approx}
        \SetAlgoVlined
        \DontPrintSemicolon
        \SetKwInOut{Input}{input}
        \SetKwInOut{Output}{output}
        \Input{A trained LM-RNN $\mathcal{M}$ of unknown rank, a target number of states $k$ \newline the size of the truncation $n>k$, a perturbation matrix $\mat{N}^n$ as in Section~\ref{sec:compactness}}
        \Output{A WFA $\widehat{A}_k^n$ of size $k$}
        Let $\widetilde{\H}^n \leftarrow$ \texttt{GetHankel($\mathcal{M},n,\mat{N}^n$)}\;
        Let $\sigma^n_k$, $\boldsymbol{\xi}^n \leftarrow$ \texttt{ComputeEigenpair($\widetilde{\H}^n$)}\;
        Let $\mat{T}^n$, $\mat{z}^n$ defined as in Equation~\ref{eq:trunc_alg}\;
        Let $\psi(z)=\frac{(\mat{z}^n)^{\top}\mat{T}^n\boldsymbol{\xi}^n}{(\mat{z}^n)^{\top}\boldsymbol{\xi}^n}$\;
        Let $r(z) \leftarrow$ \texttt{ExtractRational($\psi(z)$)}\;
        Let $\mat{G}^n_k \leftarrow$ \texttt{RecoverMatrix($r(z)$,$k+1$)}\;
        Let $\widehat{A}_k^n \leftarrow$ \texttt{SpectralMethod($\mat{G}^n_k$,$\mathcal{B}$)}\;
        \Return{$\widehat{A}_k^n$} 
    \end{algorithm}

\paragraph*{Filling the Matrix}

    Following~\citet{Ayache2018}, we consider a trained LM-RNN, and use it to fill the entries of a Hankel matrix $\H^n$. We obtain a $n\times n$ Hankel matrix $\H^n$, having entries $f_n$ on the first $n$ anti-diagonals, and zeroes everywhere else. As mentioned in Section~\ref{sec:compactness}, we add a perturbation to $\H^n$, \emph{i.e.} a random Hankel matrix of noise $\mat{N}^n$, which can be set to zero when the singular numbers $\sigma_k$ and $\sigma_{k-1}$ of $\H$ are known to be distinct. The output of \texttt{GetHankel} is the perturbed matrix $\widetilde{\H}^n=\H^n+\mat{N}^n$.
  
\paragraph*{Computing a Schmidt Pair}
    
    The function \texttt{ComputeEigenpair} returns the singular number $\sigma^n_k$ of $\widetilde{\H}^n$, and a corresponding singular vector.
    Since $\widetilde{\H}^n$ has finite rank and is symmetric, its singular numbers are the absolute values of the corresponding eigenvalues, \emph{i.e.} $\sigma^n_{k}=|\lambda_{k}|$. Analogously, given the eigenvalue $\lambda_k$ and a corresponding eigenvector $\mat{v}_k^n$, a Schmidt pair is given by $(\boldsymbol{\xi}^n,\boldsymbol{\eta}^n)$, with $\boldsymbol{\xi}^n=\mat{v}_k^n$, $\boldsymbol{\eta}^n=\operatorname{sgn}(\lambda_k)\mat{v}_k^n$, and $\operatorname{sgn}(\lambda_k)=\lambda_k/|\lambda_k|$. 
    
\paragraph*{Rational function} 
    
    From Theorem~\ref{theorem:Kronecker} we know that finite rank Hankel matrices correspond to strictly proper rational functions, with all the poles inside the complex unit disc. In order to find the best approximation, we apply Equation~\ref{eq:toeplitz} from Theorem~\ref{toeplitz:unimodular}, and obtain, before applying the projection, a function $\psi(z)=\frac{a(z)}{b(z)}$. Note that we are interested in keeping only $r(z)=\mathbb{P}_-\psi(z)$, as $\psi(z)$ might contain poles outside the unit disc. Since the poles of $\psi(z)$ correspond to the zeros of $b(z)$, we can isolate the part of the function with poles inside the unit disc using partial fraction decomposition. This method allows us to rewrite the rational function $\psi(z)=\frac{a(z)}{b(z)}$ as:
    \begin{equation}
        \psi(z)=\frac{a(z)}{b(z)}=c(z)+\sum_i\frac{a_i(z)}{b_i(z)},
    \end{equation}
    where each $\frac{a_i(z)}{b_i(z)}$ is a strictly proper rational function, and each factor $b_i$ of the denominator is a power of an irreducible polynomial. Now, we analyze the zero of each $b_i$: if it is outside or on the complex unit disc, then we discard the term $\frac{a_i(z)}{b_i(z)}$. The output of \texttt{ExtractRational} is the sum of the remaining terms, corresponding to the component in $\mathcal{H}^2_-$ of $\psi(z)$. We remark that the partial fraction decomposition can be computed efficiently, with the naive implementation having complexity $O(n^3)$ for a fraction with $n$ poles~\citep{partial_fractions}.

\paragraph*{Recovering the Matrix}
 
    In the previous step we have obtained a strictly proper rational function:
    \begin{equation}
        r(z)=\frac{p(z)}{q(z)}, \qquad \text{ where} \quad p(z)=\sum_{i=1}^k p_iz^{k-i}, \quad q(z)=z^k+\sum_{i=1}^{k} q_{i}z^{k-i},
    \end{equation}
    and $p(z)$ and $q(z)$ are relatively prime, with $q(z)$ having degree $k$.
    As seen in subsection~\ref{subsection:AAK}, if $r(z)=\sum_{n\geq 0}g_n z^{-n-1}$, then $\mat{G}^n_k(j,k) = g_{j+k}$. The coefficients $g_i$ of the Hankel matrix can be recovered from the following set of equations, obtained from the constructive proof of Theorem~\ref{theorem:Kronecker}~\citep{discreteH}:
    \begin{equation}
    \begin{cases}
        g_0 = p_1 \\
        g_1 = p_2 - g_0q_1\\
        \dots \\
        g_{k-1}= p_k-g_{k-2}q_1 - \dots - g_0q_{k-1}
    \end{cases} \quad\quad
    \begin{cases}
        g_{k} + \sum_{i=1}^k q_ig_{k-i}=0 \\
        g_{k+1} + \sum_{i=1}^k q_ig_{k+1-i}=0\\
        \dots
    \end{cases}.
    \end{equation}
    These equations form a linear system, which can be easily solved to derive the matrix $\mat{G}^n_k$ of rank $k$ having entries $\mat{G}^n_k(i,j)=g_{i+j}$. Note that to extract a WFA using the spectral method we don't actually need to compute all the coefficients of $\mat{G}$. In fact, we will show in the next paragraph that the first $k+1$ coefficients are enough to retrieve the WFA.

\paragraph*{Extracting the WFA}

    We can finally recover the minimal WFA $\widehat{A}_k^n=\wa$ with $k$ states computing the function $g: \Sigma^* \rightarrow \R$ such that $\mat{G}^n_k(i,j)=g_{i+j}$. 
    We use the spectral method outlined in subsection~\ref{subsection:WFA}. The key point of the algorithm is to select a prefix-closed and complete basis $\mathcal{B}$. As noted before, since we are working with a one-letter alphabet, the Hankel matrix $\mat{G}$ is symmetric. In this case, if $\mat{G}^n_k$ has rank $k$, then the size of the biggest leading principal submatrix is $k\times k$~\citep{minor}. Consequently, the natural choice for $\mathcal{B}=(\mathcal{P},\mathcal{S})$ is to have $\mathcal{P}=\mathcal{S}$, with $\mathcal{P}$ containing all the strings having size strictly smaller than $k$. Following the notation of Section~\ref{subsection:WFA}, $\H_{\varepsilon}$ corresponds to the $k \times k$ leading principal submatrix of $\mat{G}^n_k$, and $\mat{h}_{\mathcal{P},\varepsilon},\, \mat{h}_{\varepsilon,\mathcal{S}}$ are its first column and row, respectively. Finally, $\H_a$ is the sub-block of $\mat{G}^n_k$ having the same rows as $\H_{\varepsilon}$, and the columns obtained by shifting each individual column of $\H_{\varepsilon}$ by one column. Using Equation~\ref{eq:spectralmethod} we obtain the WFA $\widehat{A}_k^n$.

\section{Error and Convergence}
\label{sec:error}

    If the matrix of the LM-RNN has finite rank, the unique optimal approximation of size $k$ can be recovered, and the error, which can be computed using Gramian matrices, is given by $\sigma_k$~\citep{AAK-WFA}. Moreover, due to the ordering of the singular numbers, the error is guaranteed to decrease when the size of the approximation gets closer to the actual rank of the matrix. On the other hand, if the rank is not finite we can only recover an asymptotically-optimal solution, and a bound for the error. As seen in Theorem~\ref{thm:nearopt}, we can estimate how far we are from the optimal error $\sigma_k$:
    \begin{equation}
        \norm{\H-\mat{G}_k^n}\leq \sigma_k+ 2\left(1-\sum_{i=0}^{n}f(i)\right).
    \end{equation}
    We know that $f\in\ell^1$, so $f(n)\to 0$, meaning that ``little'' probability is allocated to very long strings. Thus, a direct way to reduce the error is to select the biggest possible $n$. An estimate for $\sigma_k$ in terms of $\sigma_k^n$ can be obtained using Lemma~\ref{lemma:singnumb} in Appendix~\ref{apd:first}:
    \begin{equation}
        |\sigma_k-\sigma_k^n|\leq 1-\sum_{i=0}^{n}f(i).
    \end{equation}
    An alternative way to reduce the error when additional information is available is to explore other types of truncations, to try to improve the convergence rate~\citep{chui_cont}. 
    
    If a matrix of noise $\mat{N}$ is added to $\H$ (see Section~\ref{sec:compactness}), we need to consider its effect on the error. Given the infinite matrix $\mat{N}$, we consider the matrix $\mat{N}^n$ obtained by truncation in a way analogous to Equation~\ref{eq:trunc}. We obtain the following bound. 
    
    \begin{theorem}\label{thm:noise}
        Let $\mat{N}^n$ be defined as above, and let $\widetilde{\mat{G}}_k^n$ be an asymptotically-optimal $(n,k)$-approximation of $\widetilde{\H}^n=\H^n+\mat{N}^n$. Then the error is bounded by:
    \begin{equation}
        \norm{\H-\widetilde{\mat{G}}_k^n}\leq \norm{\H-\mat{G}_k^n} + 2\norm{\mat{N}^n}.
    \end{equation}
    \end{theorem}
    
    \begin{proof}
        Let $\widetilde{\mat{G}}_k^n$ and $\widetilde{\sigma}_k^n$ be the optimal approximation and the $(k+1)$-th singular number of $\H^n+\mat{N}^n$, respectively. From Theorem~\ref{theorem:aakop} we have:
        \begin{equation}\label{eq:er_noise}
            \norm{\H^n+\mat{N}^n-\widetilde{\mat{G}}_k^n}=\widetilde{\sigma}_k^n.
        \end{equation}
        Then:
        \begin{align*}
            \norm{\H-\widetilde{\mat{G}}_k^n} &\leq \norm{\H-\H^n-\mat{N}^n}+ \norm{\H^n+\mat{N}^n-\widetilde{\mat{G}}_k^n}\\
            & \leq \norm{\H-\H^n}+\norm{\mat{N}^n}+\widetilde{\sigma}_k^n\\
            & \leq \norm{\H-\H^n}+2\norm{\mat{N}^n}+\sigma_k^n\\
            & \leq \sigma_k+ 2\norm{\H-\H^n} + 2\norm{\mat{N}^n}
        \end{align*}
        where we used Equation~\ref{eq:er_noise} in the second step, and in the last two steps we applied Lemma~\ref{lemma:singnumb}, with:
        \begin{equation*}
            |\widetilde{\sigma}_k^n-\sigma_k^n| \leq \norm{\mat{N}^n},
        \end{equation*}
        and 
        \begin{equation*}
            |\sigma_k^n-\sigma_k| \leq \norm{\mat{H}-\H^n}.\qedhere
        \end{equation*}
    \end{proof}
    This means that the additional error depends only on the norm of the matrix of noise. We already remark that this can be chosen to be arbitrarily small, and since only a finite sub-block of $\mat{N}^n$ is different from zero, the norm can be precisely computed.
    
    Finally, as noted and proved by~\citet{AAK-WFA}, the $\ell^2$-norm is bounded by the spectral norm.
    
    \begin{theorem}\label{l2bound}
        Let $f:\N \rightarrow \R$, $\H$ and $\H^n$ as before. Let $\widehat{A}_k^n$ be an asymptotically-optimal $(n,k)$-approximation computing $g:\N \rightarrow \R$, with matrix $\mat{G}_k^n$. Then:
        \begin{equation}
            \norm{f-g}_{\ell^2} \leq \norm{\H - \mat{G}_k^n}.
        \end{equation}
    \end{theorem}
    \begin{proof}
        Let $\mat{e}_0=\begin{pmatrix} 1 & 0 & \cdots \end{pmatrix}^{\top}$, $f:\N\rightarrow \R$, $g:\N \rightarrow \R$ with Hankel matrices $\H$ and $\mat{G}_k^n$, respectively. Let $\H^n$ be the truncation of $\H$ defined as before.
        
        We have:
        \begin{align*}
            \norm{f-g}_{\ell^2}=\left(\sum_{n=0}^{\infty}|f_n-g_n|^2 \right)^{1/2}&=\norm{(\H-\mat{G}_k^n) \mat{e}_0}_{\ell^2} \\
            &\leq \sup_{\norm{\mat{x}}_{\ell^2}=1}\norm{(\H-\mat{G}_k^n)\mat{x}}_{\ell^2} \\
            &\leq \norm{\H-\mat{G}_k^n}\\
            &\leq \sigma_k + 2\left(1-\sum_{i=0}^{n}f(i)\right).
        \end{align*}
        The second equation follows by definition and by observing that matrix difference is always computed entry-wise, while the last inequality is a consequence of Equation~\ref{eq:error}. 
    \end{proof}
    
    A point deserving further investigation is to understand how our approximation method performs with respect to more popular metrics (such as word error rate or normalized discounted cumulative gain). This could help evaluate how meaningful it is to use the spectral norm in an experimental setting, but the comparison is possible only for multi-letter alphabets. 

\section{Related Work}
   
    Several works in the literature analyze the relation between RNNs and WFAs. The work of~\citet{Rabusseau19} highlight a structural correspondence between WFAs and second order RNNs with linear activation function, showing that they are expressively equivalent. 
    \citet{WeissWFA19}, instead, propose a method to extract probabilistic deterministic finite automata from RNNs, based on conditional probabilities and on a local tolerance to compare observations. Analogously,~\citet{Takamasa} use spectral learning and regression methods to extract a WFA from a RNN trained on rational languages. \citet{Ayache2018} and~\citet{eyraud2020} propose a spectral algorithm to extract a WFA from a black box model for language modelling, without accessing the training samples. We remark that, while we focused on weighted automata, work in extraction has been done also in the context of deterministic finite automata, where the original RNN is used like a binary classifier~\citep{Giles,OmlinGiles,giles2,WeissDFA18}.
    
    The approximate minimization problem has been studied also for other types of models. For finite state machines,~\citet{Balle15,Balle19} and~\citet{ballerabusseau} present a technique based on the canonical expressions of weighted and weighted tree automata, respectively. \citet{AAK-WFA} use AAK theory to address the optimal spectral-norm approximate minimization problem for a large class of WFAs over a one-letter alphabet. The control theory community has studied this problem in the context of linear time-invariant systems~\citep{antoulas}. A first approximation algorithm is due to~\citet{Kung80,KungLin}, followed by state-space solutions from \citet{Glover} for the optimal continuous case. 
    Most of the results for the discrete case relies on stricter assumptions or are sub-optimal~\citep{gu,balldiscrete, Al-Hussari, Ionescu, discreteH,chuisystem_red}.
    A partial solution for the approximation problem of infinite-dimensional systems can be found in~\citet{Glover_infdim}. The extensive work of~\citet{chuisystem_red,chui_singvect,chui_cont}, that provides some of the theoretical results we used, analyzes the continuity of approximation and truncation methods in signal processing. 

\section{Conclusion}
    
    In this paper we studied the approximate minimization problem for black boxes trained for language modelling of sequential data over a one-letter alphabet. To solve this problem, we applied the AAK theory for Hankel operators~\citep{AAK71} and continuity results from the control theory literature~\citep{chui_cont,chuisystem_red}. This allowed us to extended the contribution of~\citet{AAK-WFA} to the case of infinite-rank Hankel matrices. Given a language model and a target size as input, we provided an algorithm to extract a WFA corresponding to an asymptotically-optimal approximation in the spectral norm. The algorithm can be applied to black box models like RNNs or transformers. 
    
    The use of approximate minimization over regular extraction has the advantage that it allows us to choose the size of the approximation and search the optimal WFA within this constraint. This is particularly useful when the extracted WFA is used for interpretability. In fact, every WFA has a graphical representation, but this is helpful only when the number of states is small enough to actually make it readable. Moreover, approximate minimization can be used to reduce the computational cost of the task considered, as the new model is smaller and often easier to compute compared to the original one. Note that, even though the spectral algorithm used in our approach does not guarantee that the extracted WFA will preserve the probabilistic nature of the function considered, there are methods that can partially address this issue~\citep{baillyQuadrWFA,anandkumarQuadrWFA}. 
    
    While the choice of the spectral norm to evaluate the approximation deserves further investigation, we think that it constitutes an interesting way to approach the problem of approximating black boxes with WFAs. In particular, it allows us to precisely compute the distance between different classes of models, for example RNNs and WFAs, and to (asymptotically) find the optimal approximation of a given size.
    
    The one-letter setting is certainly restrictive, but it is a first step towards developing provable approximation algorithms for black box models. In fact, it allows us to introduce AAK techniques in the context of black boxes for language modelling. The application of this rich mathematical theory has shown to be very effective in areas like control theory or signal processing, and our work highlights fruitful connections with these fields. Moreover, one-letter alphabets have proven to be of independent interest when dealing with automata, as in this case the classes of regular and of context-free languages collapse~\citep{Pighizzini}. 
    
    The natural next step for future work is to extend our results to larger alphabets. This cannot be done directly, since the correspondence with Hardy spaces holds only in the one-letter case. Even though a non-commutative version of AAK theory has been recently studied~\citep{popescu}, adapting this extension to functions on sequential data remains challenging.
    Nonetheless, we think that the strong theoretical foundations of this work, together with a provable algorithm for the approximate minimization problem and the possibility to compute the distance between different classes of models, make this direction worth pursuing.
 
\section*{Acknowledgments}
    This research has been supported by NSERC Canada (C. Lacroce, P. Panangaden) and Canada CIFAR AI chairs program (G. Rabusseau). The authors would like to thank Doina Precup for supporting this work, Borja Balle for fruitful discussions on the approximate minimization problem, and Maxime Wabartha for help with the problem formulation and for a detailed feedback.


\bibliography{bibliography}

\newpage
\appendix

\section{Example}\label{appendix:ex}
    
    In this section we show an illustrative example, analogous to the one presented by~\citet{AAK-WFA}.
    
    We consider the function $f:\N\rightarrow\R$, computing a probability, defined as:
    \begin{equation*}
            f(k)= \begin{cases}
                    0   &\text{if $k$ is odd} \\
                    \frac{8}{9}3^{-k}   &\text{if $k$ is even}
                    \end{cases} 
        \end{equation*}
    The corresponding Hankel matrix is:
    \begin{equation}
            \H=\begin{pmatrix} f(0) & f(1) & f(2) & \dots \\
                               f(1) & f(2) & f(3) &\dots \\
                               f(2) & f(3) & f(4) &\dots \\
                                \vdots& \vdots &\vdots&\ddots
            \end{pmatrix}=
            \begin{pmatrix} \frac{8}{9} & 0 & \frac{8}{81} & \dots \\
                               0 & \frac{8}{81} & 0 &\dots \\
                               \frac{8}{81} & 0 & \frac{8}{729} &\dots \\
                                \vdots& \vdots &\vdots&\ddots
            \end{pmatrix}.
    \end{equation} 
    
    If we consider the Hankel matrix with respect to the basis of the Hardy space, since $\H(j,k)= \widehat{\phi}(-j-k-1)$, we have:
    \begin{equation*}
            \H=
            \begin{pmatrix} \frac{8}{9} & 0 & \frac{8}{81} & \dots \\
                               0 & \frac{8}{81} & 0 &\dots \\
                               \frac{8}{81} & 0 & \frac{8}{729} &\dots \\
                                \vdots& \vdots &\vdots&\ddots
            \end{pmatrix}
            =\begin{pmatrix} \widehat{\phi}(-1) & \widehat{\phi}(-2) & \widehat{\phi}(-3) & \dots \\
                               \widehat{\phi}(-2) & \widehat{\phi}(-3) & \widehat{\phi}(-4) &\dots \\
                               \widehat{\phi}(-3) & \widehat{\phi}(-4) & \widehat{\phi}(-5) &\dots \\
                                \vdots& \vdots &\vdots&\ddots
            \end{pmatrix}.
        \end{equation*} 
        and the rational component of a symbol for $\H$ is: 
        \begin{equation*}
        \mathbb{P}_-\phi=\sum_{n \geq 0}\widehat{\phi}(-n-1)z^{-n-1}=\sum_{n \geq 0}\frac{8}{9}9^{-n}z^{-2n-1}=\frac{8z}{9z^2-1}.
        \end{equation*}

\section{Proofs}\label{apd:first}

    In this section, we recall a few fundamental mathematical results used in the paper.
    
    \subsection*{Nehari's Theorem}
    
    We start by Nehari's Theorem~\citep{Nehari}, which is used in the proof of Theorem~\ref{thm:wiener}. We state two versions of this theorem, which can be found in~\citet{Peller}. This result is of great importance in the theory of Hankel operator, as it highlights the correspondence between bounded Hankel operators and functions in $\mathcal{L}^{\infty}(\mathbb{T})$.
    
    \begin{theorem}[\citet{Nehari}]\label{thm:nehari1}
        Let $H: \ell^2 \rightarrow \ell^2$ be a Hankel operator having Hankel matrix $\H(j,k)=\{\alpha_{j+k}\}_{j,k \geq 0}$. We have that $H$ is bounded on $\ell^2$ if and only if there exists a function $\psi \in \mathcal{L}^{\infty}(\mathbb{T})$ such that:
        \begin{equation}
            \alpha_m=\widehat{\psi}(m), \quad m\geq 0.
        \end{equation}
        In this case:
        \begin{equation}
             \norm{H} =\inf\{\norm{\psi}_{\infty}:\widehat{\psi}(n)=\widehat{\phi}(n), \, n \geq 0\}.
        \end{equation}
        Where $\widehat{\psi}(n)$ is the $n$-th Fourier coefficient of $\psi$.
    \end{theorem}
    
    We can now reformulate this theorem using the characterization of Hankel operators in Hardy spaces.
    
    \begin{theorem}[\citet{Nehari}]\label{thm:nehari}
        Let $\phi \in  \mathcal{L}^2(\mathbb{T})$ be a symbol of the Hankel operator on Hardy spaces $H_{\phi}:\mathcal{H}^2 \rightarrow \mathcal{H}^2_-$. Then, $H_{\phi}$ is bounded on $\mathcal{H}^2$ if and only if there exists $\psi \in \mathcal{L}^{\infty}(\mathbb{T})$ such that $\widehat{\psi}(m)=\widehat{\phi}(m)$ for all $m<0$. If the conditions above are satisfied, then:
        \begin{equation}\label{eq:nehari}
            \norm{H_{\phi}}=\inf\{\norm{\psi}_{\infty}:\widehat{\psi}(m)=\widehat{\phi}(m), \, m<0\}.
        \end{equation}
    \end{theorem}

    \subsection*{Riesz Eigenvalues Inequality}
    
    The following result is used in the proof of Theorem~\ref{thm:nearopt} and to establish the relation between the singular value of the original Hankel matrix and the one of the truncation. This result is due to ~\citet{riesz} (see~\citet{krein_linearop} for the proof and for a more general version of this theorem). 
    
    \begin{lemma}[\citet{riesz}]\label{lemma:singnumb}
        Let $T,S$ be two self-adjoint compact operators, and let $\sigma_k^T$, $\sigma_k^S$ for $k \geq 0$ be their singular numbers. Then:
        \begin{equation}
            |\sigma_k^T-\sigma_k^S| \leq \norm{\mat{S}-\mat{T}}.
        \end{equation}
    \end{lemma}
    
    We remark that, while the original statement is about eigenvalues, in our setting this holds automatically for singular values.

    \subsection*{Cauchy's Interlace Theorem}
    
    Cauchy's Interlace Theorem is used in the proof of Theorem~\ref{thm:nearopt}. We refer the reader to~\citet{cauchy} for a proof this theorem.
    
    \begin{theorem}[Cauchy's Interlace Theorem]\label{thm:cauchy}
        Let $\mat{A}$ be a $n\times n$ Hermitian matrix, let $\mat{B}$ be the principal submatrix of $\mat{A}$ of order $(n-1)\times(n-1)$. If $\lambda_0\geq\dots\geq\lambda_{n-1}$ are the eigenvalues of $\mat{A}$, and $\mu_0\geq\dots\geq\mu_{n-2}$ are the eigenvalues of $\mat{B}$, then:
        \begin{equation}
            \lambda_0\geq\mu_0\geq\lambda_1\geq\mu_1\geq\dots\lambda_{n-2}\geq\mu_{n-2}\geq\lambda_{n-1}.
        \end{equation}
    \end{theorem}

   \subsection*{Rank of the Principal Submatrix}
    
    We conclude by recalling the result of~\citet{minor}, which shows the relation between the rank of the Hankel matrix and the one of its leading principal submatrix.
    
    \begin{theorem}[\citet{minor}]
        A Hankel matrix $\H$ has a finite rank $r$ if and only if the first $r$ rows of $\H$ are linearly independent, and generate the row $r+1$ as a linear combination.
    \end{theorem}

\end{document}